\documentclass[letterpaper, 10 pt, conference]{ieeeconf}  

\IEEEoverridecommandlockouts                           
\usepackage{cite}
\usepackage{times}
\usepackage{multicol}
\usepackage[bookmarks=true]{hyperref}
\usepackage{amsmath}

\usepackage{amsthm}
\usepackage{amsfonts}
\usepackage{amssymb}
\usepackage{mathtools}
\usepackage{arydshln}
\usepackage{booktabs,caption}
\usepackage[flushleft]{threeparttable}
\usepackage{graphicx}
\usepackage{subcaption}
\usepackage{caption}
\usepackage{multirow}
\usepackage{accents}
\usepackage[algo2e,algoruled,boxed,lined,norelsize]{algorithm2e}
\graphicspath{ {figures/} }
\usepackage[bookmarks=true]{hyperref}

\newcommand{\ubar}[1]{\underaccent{\bar}{#1}}

\newcommand{\diag}{{\rm diag\;}}
\newcommand{\col}{{\rm col\;}}

\DeclarePairedDelimiter{\norm}{\lVert}{\rVert}

\DeclareMathOperator{\rank}{rank}

\newtheorem{definition}{Definition}
\newtheorem{lemma}{Lemma}

\title{\LARGE \bf
Learning Objective Functions Incrementally by Inverse Optimal Control}

\begin{document}

\author{Zihao Liang$^{*}$, Wanxin Jin$^{*}$, Shaoshuai Mou$^{*}$

\thanks{$^{*}$The authors are with the School of Aeronautics and Astronautics, Purdue University, 47906 IN, USA.
        Email: {\tt\small liang331@purdue.edu, wanxinjin@gmail.com, mous@purdue.edu}}%
}

\maketitle
\thispagestyle{empty}
\pagestyle{empty}

%%%%%%%%%%%%%%%%%%%%%%%%%%%%%%%%%%%%%%%%%%%%%%%%%%%%%%%%%%%%%%%%%%%%%%%%%%%%%%%%
\begin{abstract}

% This paper proposes an inverse optimal control method which enables an agent to incrementally learn an objective function from a collection of segment trajectories . The unknown objective function is parameterized as a weighted sum of features with unknown weights. Each segment trajectory is a small part of optimal trajectory within an interval of the time horizon. Different from existing methods, we propose a permissive criterion to evaluate the informativeness of a segment, allowing a much shorter segment, which can even be a single trajectory point, to be used in the learning process. The proposed method shows that each trajectory segment, if informative, can pose a linear constraint to the unknown weights, thus, the objective function can be learned by incrementally incorporating all informative segments.  Effectiveness of the method is shown on a simulated 2-link robot arm and a 6-DoF maneuvering quadrotor system, in each of which only small demonstration segments are available.

This paper proposes an inverse optimal control method which enables a robot to incrementally learn a control objective function from a collection of trajectory segments. By saying incrementally, it means that the collection of trajectory segments is enlarged because  additional segments are provided as time evolves. The unknown objective function is parameterized as a weighted sum of features with unknown weights. Each trajectory segment is a small snippet of optimal trajectory.  The proposed method shows that each trajectory segment, if informative, can pose a linear constraint to the unknown weights, thus, the objective function can be learned by incrementally incorporating all informative segments.  Effectiveness of the method is shown on a simulated 2-link robot arm and a 6-DoF maneuvering quadrotor system, in each of which only small demonstration segments are available.

% This paper develops an inverse optimal control method to learn a control objective function from a collection of  demonstration segments. The unknown objective function is parameterized as a weighted sum of features with unknown weights. Each demonstration segment is a small part of optimal trajectory  within an interval of the time horizon. Different from existing methods, we propose a permissive criterion to evaluate the informativeness of a segment, allowing a much shorter segment, which can even be a single trajectory point, to be used in the learning process. The proposed method shows that each trajectory segment, if informative, can pose a linear constraint to the unknown weights, thus, the objective function can be learned by incrementally incorporating all informative segments.  Effectiveness of the method is shown on a simulated 2-link robot arm and a 6-DoF maneuvering quadrotor system, in each of which only small demonstration segments are available.

\end{abstract}

%%%%%%%%%%%%%%%%%%%%%%%%%%%%%%%%%%%%%%%%%%%%%%%%%%%%%%%%%%%%%%%%%%%%%%%%%%%%%%%%
\section{Introduction}

In recent years, the advancements in robotics and computation capability has empowered robots to perform certain tasks by mimicking human users' behavior. Besides accomplishing complex tasks, the robots are also required to complete it in a way that a human user prefers. However, in real world applications, different human users have different preferences. As it is cumbersome to program various of preferences into robots, it is necessary to develop a method for a robot to learn human users' desired performances. This opens the door for the research in imitation learning, which allows a robot to acquire skills or behavior by observing demonstration from an expert in a given task.

With the capability of recovering an  objective function of an optimal control system from observations of the system's trajectories, inverse optimal control (IOC) has been widely applied in
learning from demonstrations \cite{abbeel2004apprenticeship,jin2020learningsparse}, where a learner  mimics an expert by learning the expert's  underlying objective function,  autonomous vehicles \cite{kuderer2015learning}, where  human driver's driving preference is learned and  transferred to vehicle controllers,  and 
human-robot interactions \cite{mainprice2016goal,byeon2021human,jin2020learningdirection},  where an objective function of  human motor control is inferred to enable efficient prediction and  coordination.

Existing IOC methods usually assume the unknown objective function could be parameterized as a linear combination of selected features (or basis functions) \cite{ng2000algorithms,mombaur2010human}. Here, each feature characterizes one aspect of the performance of the system operation, such as energy cost, time consumption,  risk levels, etc. Then, the goal of IOC   becomes   estimating the unknown weights for those features \cite{jin2019inverse}.  The authors of \cite{ratliff2006maximum,ziebart2008maximum,ziebart2009planning,jin2019pontryagin,jin2021safe} have adopted a double-layer architecture, where the estimate of the weights is updated in an outer layer while the corresponding optimal trajectory is generated by solving the optimal control problem in an inner layer.  Techniques based on the double-layer framework usually suffer high computational cost since optimal control problems need to be solved repeatedly \cite{jin2018inverse}.  Recent IOC techniques have been developed by leveraging optimality conditions,  which the observed optimal trajectory must satisfy,  and thus the unknown weights can be directly obtained by  solving  the established optimality equations.   Related work along this direction includes \cite{keshavarz2011imputing,puydupin2012convex,englert2017inverse}, where  Karush-Kuhn-Tucker conditions are used, \cite{molloy2018finite}, where  Pontryagin's maximum principle \cite{pontryagin1962mathematical} are used.

Despite significant progress achieved as described above, most existing  IOC methods cannot learn the objective function unless  a complete system  trajectory within an entire time horizon is observed. Such requirement of  observations has limited their capabilities in disjointed or sparse segments of a complete trajectory. Apart from our problem setting, in some applications where only incomplete trajectory data is available, for example, due to limited sensing capability, sensor failures, or occlusion \cite{bogert2016expectation,bajcsy2017learning}, the method of IOC does not guarantee to retrieve an accurate representation of the objective function. In \cite{bajcsy2017learning}, given sparse corrections (demonstrations),  the authors  create an intended trajectory of full horizon based on the sparse data  using trajectory shaping/interpolation \cite{dragan2015movement}, in order to utilize the maximum margin IOC approach \cite{ratliff2006maximum}. Although successful in learning from human corrections, it is likely that the artificially-created trajectory might not exactly reflect the actual trajectory of a human expert. In \cite{bogert2017scaling}, the authors model the missing data using a  probability distribution, then both the objective function and the missing part are learned under the maximization-expectation  framework. Besides  huge computational cost,  this work, however, has not provided how percentage of missing information affects  learning performance. 

In recognition of the above limitations, this paper aims to develop an approach to learn the objective function  incrementally from available trajectory segments. By saying \emph{trajectory segments}, we refer to a collection of segments of the system's trajectory of states and inputs in any time intervals of the horizon; we  allow a  segment to be a single data point, i.e., a state/input at a single  time instant. Each segment may not be sufficient to determine the objective function by itself, an incremental approach will be developed to incorporate all available segments to achieve an estimate of the unknown  weights of the objective function. 

% , without the attempt to characterize  missing information. Our recent work \cite{jin2019inverse} addressed similar problems, where a notion of the recovery matrix has been introduced to solve IOC using incomplete observations. However, it still requires the observation data to be consecutive and long enough to satisfy the recovery condition, and cannot utilize a smaller segment of trajectory. Different from our previous work \cite{jin2019inverse} , the method proposed in this paper permits a smaller segment of demonstrations that may not necessarily satisfy such rank condition. 

\subsection*{Notations}
The column operator $\col \{\boldsymbol{x}_1,\boldsymbol{x}_2,...,\boldsymbol{x}_k\}$ stacks its (vector) arguments into a column.
$\boldsymbol{x}_{k_1:k_2}$ denotes a stack of multiple $\boldsymbol{x}$ from $k_1$ to $k_2$ ($k_1\leq k_2$), that is,  $\boldsymbol{x}_{k_1:k_2}=\col \{\boldsymbol{x}_{k_1},...\boldsymbol{x}_{k_2}\}$.
$\boldsymbol{A}$ (bold-type) denotes a block matrix.
Given a vector function $\boldsymbol{f}(\boldsymbol{x})$ and a constant $\boldsymbol{x}^*$, $\frac{\partial \boldsymbol{f}}{\partial \boldsymbol{x}^*}$ denotes the Jacobian matrix with respect to $\boldsymbol{x}$ evaluated at $\boldsymbol{x}^*$.
Zero matrix/vector is denoted as $\boldsymbol{0}$, and identity matrix as ${\boldsymbol{I}}$, both with appropriate dimensions. $\boldsymbol{A}^\prime$ denotes the transpose of matrix $\boldsymbol{A}$.

\section{Problem Statement}
Consider the following discrete-time dynamical system\footnote{In (\ref{dynamics}), at time $k$, we denote the control input as $\boldsymbol u_{k+1}$ instead of $ \boldsymbol u_{k}$ due to notation simplicity of following expositions, as adopted in \cite{levine2012continuous}.}:
\begin{equation} \label{dynamics}
\boldsymbol x_{t+1}=\boldsymbol f(\boldsymbol x_t, \boldsymbol u_{t+1}), \quad \boldsymbol{x}_0 \in \mathbb{R}^n,
\end{equation}
where  vector function $\boldsymbol f: \mathbb{R}^n \times\mathbb{R}^m \mapsto \mathbb{R}^n $ is  differentiable; $\boldsymbol x_t \in \mathbb{R}^n$ denotes the system state;  $\boldsymbol u_t \in \mathbb{R}^m$ is the control input; and $t=0,1,\dots$ is the time step. Let  
\begin{equation}\label{eq_data}
\boldsymbol{\xi}=\{\boldsymbol{\xi}_t:  t=1,2,...,T\}\,\, \text{ with } \,\,\boldsymbol{\xi}_t=\{\boldsymbol{x}^*_t, \boldsymbol{u}^*_t\}
\end{equation} 
 denote a trajectory of system states and inputs in a time horizon $T$. Note that here $T$ could be chosen arbitrarily large or even the infinity. Suppose the   system trajectory $\boldsymbol{\xi}$ is a result of optimizing the following objective function:
\begin{equation} \label{objFunc}
J(\boldsymbol{x}_{1:T},\boldsymbol{u}_{1:T})=\sum\nolimits_{t=1}^{T}{\boldsymbol{\omega}}^\prime \boldsymbol{\phi}(\boldsymbol{x}_t,\boldsymbol{u}_t).
\end{equation}
Here,  $\boldsymbol{\phi}:\mathbb{R}^n\times\mathbb{R}^m\mapsto\mathbb{R}^r$ is a vector of specified features  (or basis functions), with each feature  $\phi_i$ differentiable; and $\boldsymbol{\omega}\in  \mathbb{R}^r$ is the unknown weight vector with the $i$th element $\omega_i$ being the weight for feature $\phi_i$, $1\leq i\leq r$. 

Suppose that at each time step, one is accessible to an additional data segments (i.e. the collection of available data segments is enlarged as time evolves). 
% % \emph{trajectory segments}
% , denoted by $\mathcal{S}$, which is a set of data segments of $\boldsymbol{\xi}$, and $\mathcal{S}\subseteq \boldsymbol{\xi}$.  
A  \emph{data segment} is defined as  a  sequence of system states and inputs $\boldsymbol{\xi}_{\ubar{t}:\bar{t}}\subseteq \boldsymbol{\xi}$,   where $\ubar{t}$ and $\bar{t}$ denote the starting and end time of such  segment, respectively, and $1\leq \ubar{t}\leq \bar{t} \leq T$. The set of data segments at time step $t$, denoted by $\mathcal{S}_{t}$, is defined as:
\begin{align}
\mathcal{S}_{t}&\triangleq\{ \boldsymbol{\xi}_{\ubar{t}_{i}:\bar{t}_{i}}: i=1,2,\cdots,N\}. \label{segSetEachTime}
\end{align}
where $\ubar t_{i}$ and $\bar t_{i}$ are the starting and end time of the $i$th data segment. $N$ here is used to denote the total number of the available segments at time step $t$ and it becomes larger over time as more data segments are available. It is worth noting that we do not put any restrictions on  $\mathcal{S}_{t}$, which means that any segment in it can  be the full trajectory $\boldsymbol{\xi}$ or  even a single  input-state point at a time instance in terms of $\ubar{t}_{i}=\bar{t}_{i}$. Different segments are also allowed to have overlaps. 
% Also, in the method developed below, we do not require the knowledge  $\ubar{t}_{i}$, i.e., the starting time of each segment relative to the starting time of the system trajectory.
% We further define the collection of available data segments $\mathcal{S}$,
% \begin{align}
% \mathcal{S}&\triangleq\{ \mathcal{S}_{k}: k=1,2,\cdots,K\}. \label{segSetAgentj1}
% \end{align}
% where $K$ denotes the number of time steps that provide data segments, and $K\le T$.

Since the set of segments $\mathcal{S}_t$ at each time step may not be sufficient to determine $\boldsymbol{\omega}$ by itself, thus \textbf{the problem of interest} is to develop an algorithm to incrementally estimate $\boldsymbol{\omega}$  via IOC by incorporating data segments provided in set $\mathcal{S}_t$.  

\section{The Proposed Approach}
In this section, we first present the idea  of how to establish a constraint on the feature weights from any available segment data, then develop the incremental IOC approach.

\subsection{Key Idea to Utilize Any Trajectory Segment  in IOC}

Let $\boldsymbol{\xi}_{\ubar{t}:\bar{t}}$ be any  segment of the full trajectory $\boldsymbol{\xi}$  (\ref{eq_data})  with $1\leq \ubar{t}\leq \bar{t} \leq T$. Since the full trajectory $\boldsymbol{\xi}$ is generated by the system  (\ref{dynamics})  minimizing  (\ref{objFunc}),  when considering an infinite-horizon optimal control setting (i.e. $T$ is infinity), the trajectory is characterized by the Bellman optimality condition \cite{bertsekas2012dynamic}:
\begin{equation} \label{eq:Bellman}
    V_t(\boldsymbol{x}_t^*)=\boldsymbol{\omega}^\prime\boldsymbol{\phi(\boldsymbol{x}_t^*,\boldsymbol{u}_t^*)}+V_t(\boldsymbol{f}(\boldsymbol x_t^*, \boldsymbol u_{t}^*)),
\end{equation}

where $V_t(\boldsymbol{x}_t^*)$ is the unknown optimal cost-to-go function evaluated at state $x_t^*$. Then, if we take the derivatives of (\ref{eq:Bellman}) with respect to $\boldsymbol{x}_t^*$ and $\boldsymbol{u}_{t}^*$ while denoting $\boldsymbol{{\lambda}}^*_{t}=\frac{\partial V_t(\boldsymbol{x}_t)}{\partial \boldsymbol x^*_{t}}$, we will get
% there exist a sequence of Lagrange multipliers (or costates) $\boldsymbol{ \lambda}_{1:T}^*\triangleq\col\{\boldsymbol\lambda_{1}^*,\cdots,\boldsymbol\lambda_{T}^*\}$ such that following KKT optimality conditions \cite{bertsekas1997nonlinear} hold for $\boldsymbol{\xi}$, that is,
% \begin{align}\label{KKT} 
% \frac{\partial L}{\partial \boldsymbol{x}_{1:T}^*}=\boldsymbol{0}\quad \text{and}\quad
% \frac{\partial L}{\partial \boldsymbol{u}_{1:T}^*}=\boldsymbol{0},
% \end{align}
% where 
% \begin{equation}\label{lagrange}
% L=J(\boldsymbol{x}_{1:T},\boldsymbol{u}_{1:T})+\sum_{t=1}^{T}\boldsymbol{\lambda}^*_{t}\big(\boldsymbol{f}(\boldsymbol{x}_{t-1},\boldsymbol{u}_t)-\boldsymbol{x}_t\big)
% \end{equation}
% is  Lagrangian of the optimization (optimal control) problem. From (\ref{KKT}), one has the following equations   for any $0\leq t\leq T$:
\begin{align} 
\boldsymbol{{\lambda}}^*_{t}-\frac{\partial \boldsymbol f^\prime}{\partial \boldsymbol x^*_{t}}\boldsymbol{{\lambda}}^*_{t+1}-\frac{\partial \boldsymbol \phi^\prime}{\partial \boldsymbol x^*_{t}}\boldsymbol{\omega}&=\boldsymbol{0}, \label{pontryagin1} \\
\frac{\partial \boldsymbol f^\prime}{\partial \boldsymbol u_{t}^*}\boldsymbol{{\lambda}}^*_{t}+ \frac{\partial \boldsymbol \phi^\prime}{\partial  \boldsymbol u^*_{t}}\boldsymbol{\omega}&=\boldsymbol{0},\label{pontryagin2}
\end{align} which can also be achieved based on Pontryagin's maximum principle \cite{pontryagin1962mathematical}. It follows   that for any trajectory segment $\boldsymbol{\xi}_{\ubar{t}:\bar{t}}$, by stacking (\ref{pontryagin1})-(\ref{pontryagin2}) for  $\ubar{t}\leq t\leq \bar{t}$ one has
\begin{align}
\boldsymbol{A}(\boldsymbol{\xi}_{\ubar{t}:\bar{t}}) \boldsymbol{{\lambda}}^*_{\ubar{t}:\bar{t}}-\boldsymbol{M}(\boldsymbol{\xi}_{\ubar{t}:\bar{t}}) \boldsymbol{\omega}&=\boldsymbol{V}(\boldsymbol{\xi}_{\ubar{t}:\bar{t}})\boldsymbol{\lambda}_{\bar{t}+1}^*, \label{pontryaginCompact1} \\
\boldsymbol{B}(\boldsymbol{\xi}_{\ubar{t}:\bar{t}}) \boldsymbol{\boldsymbol{{\lambda}}}^*_{\ubar{t}:\bar{t}}+\boldsymbol{N}(\boldsymbol{\xi}_{\ubar{t}:\bar{t}}) \boldsymbol{\omega}&=\boldsymbol{0}, \label{pontryaginCompact2}
\end{align}
with
\begin{align}
\boldsymbol{A}&\triangleq
\resizebox{0.45\columnwidth}{!}{$
	\begin{bmatrix}
	\boldsymbol I &\,\,\frac{-\partial \boldsymbol{f}^\prime}{\partial \boldsymbol{x}^*_{\ubar{t}}} &                      &              \\
	\boldsymbol 0& \boldsymbol I & \ddots &   \\
	&		 & \ddots& \phantom{-1}\frac{-\partial \boldsymbol{f}^\prime}{\partial \boldsymbol{x}^*_{\bar{t}-1}}   \\
	\boldsymbol 0& & \cdots& \boldsymbol I
	\end{bmatrix}
	$}, \,\,
\resizebox{0.23\columnwidth}{!}{$
	\boldsymbol{M}\triangleq\scriptscriptstyle
	\setlength\arraycolsep{1.0 pt}
	\begin{bmatrix}
	\frac{\partial \boldsymbol\phi^\prime}{\partial \boldsymbol x^*_{\ubar{t}}} \\
	\frac{\partial \boldsymbol\phi^\prime}{\partial \boldsymbol x^*_{\ubar{t}+1}} \\
	\vdots\\
	\frac{\partial \boldsymbol\phi^\prime}{\partial \boldsymbol x^*_{\bar{t}}}
	\end{bmatrix}
	$},
\label{recMatAM}
\\
\boldsymbol{B}&\triangleq
\resizebox{0.45\columnwidth}{!}{$
	\begin{bmatrix}
	\frac{\partial \boldsymbol{f}^\prime}{\partial \boldsymbol{u}^*_{\ubar{t}}} &   &     &                           \\
	&\frac{\partial \boldsymbol{f}^\prime}{\partial \boldsymbol{u}^*_{\ubar{t}+1}} & & \\
	&	 &\ddots	 &   \\
	& & & \frac{\partial \boldsymbol{f}^\prime}{\partial \boldsymbol{u}^*_{\bar{t}}}  \\
	\end{bmatrix}
	$},\,\,
\resizebox{0.25\columnwidth}{!}{$
	\boldsymbol{N} \triangleq
	\begin{bmatrix}
	\frac{\partial \boldsymbol\phi^\prime}{\partial \boldsymbol u^*_{\ubar{t}}} \\
	\frac{\partial \boldsymbol\phi^\prime}{\partial \boldsymbol u^*_{\ubar{t}+1}} \\
	\vdots \\
	\frac{\partial \boldsymbol\phi^\prime}{\partial \boldsymbol{u}^*_{\bar{t}}}
	\end{bmatrix}$}, \label{recMatBN}
\scriptstyle
\end{align}
and  
\begin{equation} \label{recMatV}
\boldsymbol{V}(\boldsymbol{\xi}_{\ubar{t}:\bar{t}})\triangleq\col\{\boldsymbol{0}, \frac{\partial\boldsymbol{f}^\prime }{\partial \boldsymbol{x}^*_{\bar{t}}}  \}.
\end{equation}
Dimensions of the above  matrices are $\boldsymbol{A}\in \mathbb{R}^{n(\bar{t}\text{-}\ubar{t}\text{+}1)\times n(\bar{t}\text{-}\ubar{t}\text{+}1)}$,  $ \boldsymbol{B}\in \mathbb{R}^{m(\bar{t}\text{-}\ubar{t}\text{+}1)\times n(\bar{t}\text{-}\ubar{t}\text{+}1)}$,  $ \boldsymbol{M}\in \mathbb{R}^{n(\bar{t}\text{-}\ubar{t}\text{+}1)\times r}$,   $ \boldsymbol{N}\in \mathbb{R}^{m(\bar{t}\text{-}\ubar{t}\text{+}1)\times r}$, and $\boldsymbol{V}\in \mathbb{R}^{n(\bar{t}\text{-}\ubar{t}\text{+}1)\times n}$,  respectively. In above (\ref{pontryaginCompact1}), since $\boldsymbol{ \lambda}_{T+1}^*$  is undefined when $\bar{t}=T$,   we  define $\boldsymbol{ \lambda}_{T+1}^*=\boldsymbol{0}$ . It is stated in \cite{jin2019inverse} that the finite-horizon optimal control setting which has cost function and dynamics constraint expressed in Lagrangian equation yields the same equations (\ref{pontryaginCompact1}) and (\ref{pontryaginCompact2}).

Since the matrix $\boldsymbol{A}(\boldsymbol{\xi}_{\ubar{t}:\bar{t}})$ is non-singular, one can eliminate $\boldsymbol{{\lambda}}^*_{\ubar{t}:\bar{t}}$ by combining (\ref{pontryaginCompact1}) and (\ref{pontryaginCompact2}) and obtain
\begin{equation}\label{recoveryEqu}
\boldsymbol{F}(\boldsymbol{\xi}_{\ubar{t}:\bar{t}})\boldsymbol{\omega}+\boldsymbol{E}(\boldsymbol{\xi}_{\ubar{t}:\bar{t}})\boldsymbol{\lambda}^*_{\bar{t}\text{+}1}=\boldsymbol{0}.
\end{equation}
Here
\begin{align}
\boldsymbol{F}(\boldsymbol{\xi}_{\ubar{t}:\bar{t}})&=
\boldsymbol{B}{\boldsymbol{A}}^{\text{-}1}\boldsymbol{M}\text{+}\boldsymbol{N}\in\mathbb{R}^{m(\bar{t}\text{-}\ubar{t}\text{+}1)\times r},\label{recoveryMat1} \\
\boldsymbol{E}(\boldsymbol{\xi}_{\ubar{t}:\bar{t}})&=
\boldsymbol{B}{\boldsymbol{A}}^{\text{-}1}\boldsymbol{V}\in \mathbb{R}^{m(\bar{t}\text{-}\ubar{t}\text{+}1)\times n}. \label{recoveryMat2}
\end{align}
Note that (\ref{recoveryEqu}) establishes a relation between any data segment $\boldsymbol{\xi}_{\ubar{t}:\bar{t}}$,  the unknown weights $\boldsymbol{\omega}$, and the costate $\boldsymbol{\lambda}^*_{\bar{t}+1}$. Note that $\boldsymbol{\lambda}_{\bar{t}+1}$ is unknown and actually related to the  value function of future information \cite{jin2018inverse}. 

% In order to further eliminate $\boldsymbol{\lambda}^*_{\bar{t}+1}$ and measure the contribution of each data segment $\boldsymbol{\xi}_{\ubar{t}:\bar{t}}$ to solving $\boldsymbol{\omega}$, we introduce the following concept of \emph{data effectiveness} for IOC problems.

\begin{definition}[Effective Data for IOC]\label{SDEDef}
	Given  system   (\ref{dynamics}) and an arbitrary  segment  $\boldsymbol{\xi}_{\ubar{t}:\bar{t}}=\{\boldsymbol{x}^*_{\ubar{t}:\bar{t}},\boldsymbol{u}^*_{\ubar{t}:\bar{t}}\}\subset\boldsymbol{\xi}$, $1\leq t\leq \bar{t}\leq T$, we say the  segment $\boldsymbol{\xi}_{\ubar{t}:\bar{t}}$ is  data effective if
	\begin{equation}\label{SDE}
	\rank\boldsymbol{E}(\boldsymbol{\xi}_{\ubar{t}:\bar{t}})=n,
	\end{equation}
	where $\boldsymbol{E}(\boldsymbol{\xi}_{\ubar{t}:\bar{t}})\in\mathbb{R}^{m(\bar{t}\text{-}\ubar{t}\text{+}1)\times n} $ is as defined in (\ref{recoveryMat2}).
\end{definition}

It follows from Definition \ref{SDEDef} that for any effective  segment $\boldsymbol{\xi}_{\ubar{t}:\bar{t}}$, the corresponding quantity ${\boldsymbol{E}}^\prime\boldsymbol{E}$ is non-singular. Thus by multiplying ${\boldsymbol{E}}^\prime$ to both sides of (\ref{recoveryEqu}), we will have
% one can solve
% \begin{equation}\label{recoveryLambda}
% \boldsymbol{ \lambda}^*_{\bar{t}\text{+}1}=-\big({\boldsymbol{E}}^\prime\boldsymbol{E}\big)^{\text{-}1}\boldsymbol{E}^\prime\boldsymbol{F}\boldsymbol{\omega},
% \end{equation}
% which together with (\ref{recoveryEqu}) lead to
\begin{equation}\label{keyEqu}
\boldsymbol{R}(\boldsymbol{\xi}_{\ubar{t}:\bar{t}})\boldsymbol{\omega}=\boldsymbol{0}
\end{equation} with
\begin{equation}\label{keyMat}
{\boldsymbol{R}(\boldsymbol{\xi}_{\ubar{t}:\bar{t}})\triangleq\boldsymbol{F}-\boldsymbol{E}\big({\boldsymbol{E}}^\prime\boldsymbol{E}\big)^{\text{-}1}\boldsymbol{E}^\prime\boldsymbol{F}.
}
\end{equation}
Then, we have the following lemma.
\begin{lemma}\label{lemma1}
	\cite{WanxinAutomatica} For any  segment $\boldsymbol{\xi}_{\ubar{t}:\bar{t}}\subseteq\boldsymbol{\xi}$ that is data effective, $\boldsymbol{\omega}$ must satisfy (\ref{keyEqu}).
\end{lemma}

% in the sense of Definition \ref{SDEDef},

Lemma \ref{lemma1} bridges between any data-effective  segment and the unkonwn objective function weights; that is,  any effective segment enforces a set of linear constraints to weights $\boldsymbol{\omega}$. Thus, more data-effective segments result in more  constraints for recovering $\boldsymbol{\omega}$.

\subsection{Incremental IOC from Demonstration Segments}
Based on Lemma \ref{lemma1}, at each time step $t$, given a collection of $N$ data segments $\mathcal{S}_{t}=\{ \boldsymbol{\xi}_{\ubar{t}_{i}:\bar{t}_{i}}: i=1,2,...,N\}$ in (\ref{segSetEachTime}), one has
\begin{equation}
\boldsymbol{R}(\boldsymbol{\xi}_{\ubar{t}_{i}:\bar{t}_{i}})\boldsymbol{{\omega}}=\boldsymbol{0}
\end{equation}
for each segment $\boldsymbol{\xi}_{\ubar{t}_{i}:\bar{t}_{i}}$ if it is effective, where $\boldsymbol{R}(\boldsymbol{\xi}_{\ubar{t}_{i}:\bar{t}_{i}})$ is defined in (\ref{keyMat}). Then for all data-effective  segments in $\mathcal{S}_{k}$, one has the linear equation of the  weights:
\begin{subequations}
\begin{equation}
\boldsymbol{R}(\mathcal{S}_{t})\boldsymbol{\omega}=\boldsymbol{0},
\label{recEqu.1}
\end{equation}
\text{with}
\begin{equation} 
\begin{aligned}
     \boldsymbol{R}(\mathcal{S}_{t})&= \\
     \big\{\col&\{\boldsymbol{R}(\boldsymbol{\xi}_{\ubar{t}_{i}:\bar{t}_{i}})\}: \rank \boldsymbol{E}(\boldsymbol{\xi}_{\ubar{t}_{i}:\bar{t}_{i}})=n, 1\leq i \leq N \big\}
     \label{recEqu.2}
\end{aligned}
\end{equation}
\end{subequations}
Here $\boldsymbol{R}(\mathcal{S}_{t})$ is a stack of $\boldsymbol{R}(\boldsymbol{\xi}_{\ubar{t}_{i}:\bar{t}_{i}})$ for which the corresponding segment $\boldsymbol{\xi}_{\ubar{t}_{i}:\bar{t}_{i}}$ is effective.

In implementation, since the observation noise  and/or sub-optimality  exist,  directly computing  the  weights $\boldsymbol{\omega}$ from (\ref{recEqu.2})  thus may only lead to trivial solutions. Therefore, as adopted in previous IOC methods \cite{keshavarz2011imputing,puydupin2012convex,molloy2018finite,englert2017inverse}, one can choose to obtain a least square estimate for the weights by solving the following equivalent optimization,
\begin{subequations}\label{ioc_least}
\begin{equation}\label{eq_least}
\boldsymbol{\hat{\omega}}=\arg\min_{\boldsymbol{\omega}} \frac{1}{2}\norm{\boldsymbol{R}(\mathcal{S}_t)\boldsymbol{\omega}}^2,
\end{equation}
subject to 
\begin{equation}
\label{eq_cons}[ 1, 0, \cdots, 0] \, \boldsymbol{\omega}=1.
\end{equation}
\end{subequations}
Here, $\norm{\cdot}$ stands for the $l_2$ norm; and $\boldsymbol{\hat{\omega}}$ is called a least-square estimate to the unknown weights $\boldsymbol{\omega}$. Note that scaling $\boldsymbol\omega$ by an non-zero constant does not affect the IOC problem because a scaled $\boldsymbol\omega$ will result in the same trajectory $\boldsymbol{\xi}$. Without losing any generality, one can always scale $\boldsymbol\omega$ such that its first entry is equal to 1, as adopted in \cite{keshavarz2011imputing}, namely,
\begin{equation}\label{eq_omega}
\boldsymbol{e}_1^\prime\boldsymbol\omega=1\,\, \,\text{with}\,\,\,\boldsymbol{e}_1=[1, 0, \cdots, 0]^\prime\in\mathbb{R}^r.
\end{equation}

Based on the formulation in (\ref{eq_least}), if we consider at each time step $t$, an additional segment is given and added to the set $\mathcal{S}_t$. Then, as time evolves, the set of segments $\mathcal{S}_t$ is enlarged incrementally, the following lemma presents an incremental way to solve for the least square estimate $\boldsymbol{\hat{\omega}}$.

\begin{lemma}\label{lemma4}
	Given a set of trajectory segments $\mathcal{S}_t$ at time step $t$, for the $i$th segment $\boldsymbol{\zeta}_{\ubar{t}_i:\bar{t}_i}$, $1\leq i \leq N$ in the set, let
	\begin{equation} \label{eq:W_matrix}
	{W}_i^t=
	\begin{cases}
	{W}_{i\text{-}1}^t+\boldsymbol{R}(\boldsymbol{\xi}_{\ubar{t}_{i}:\bar{t}_{i}})^\prime\boldsymbol{R}(\boldsymbol{\xi}_{\ubar{t}_{i}:\bar{t}_{i}}) &\mbox{if    $\boldsymbol{\zeta}_{\ubar{t}_i:\bar{t}_i}$ effective},  \\
	{W}_{i\text{-}1}^t+\boldsymbol{0}  &\mbox{otherwise},
	\end{cases}
	\end{equation}
	with ${W}_0^t=\boldsymbol{0}$ and $\boldsymbol{R}(\boldsymbol{\xi}_{\ubar{t}_{i}:\bar{t}_{i}})$   defined in (\ref{keyMat}).   Then, we will have the matrix $W^t_N$ that is obtained with available effective data segments at current time step.
% 	sum up the $W^k_N$ matrices from all of the available time steps $K$
% 	\begin{equation} \label{W_K}
% 	    W_K = \sum_{k=1}^K W^k_N.
% 	\end{equation}
	
	As a result, the least-square estimate $\hat{\boldsymbol{\omega}}$  in (\ref{eq_least}) given previous $N$ segments is
	\begin{equation}\label{recoveredsolution}
	\boldsymbol{\hat{\omega}}=\frac{({W}^t_N)^{{-}1}\boldsymbol{e}_1}{\boldsymbol{e}_1^\prime ({W}^t_N)^{{-}1}\boldsymbol{e}_1}.
	\end{equation}
\end{lemma}
\begin{proof}
Consider  $\boldsymbol{R}(\mathcal{S}_t)$ in (\ref{recEqu.2}), with the matrix ${W}^t_N$ described in (\ref{eq:W_matrix}), the optimization problem in (\ref{ioc_least}) is equivalent to
	\begin{equation}\label{quadprog1}
	\min_{\boldsymbol{\omega}}
		\boldsymbol{\omega}^\prime
		W^t_N
		\boldsymbol{\omega}
 \quad\,\, \text{s.t.} \quad \boldsymbol{\omega}^\prime \boldsymbol{e}_1=1.
	\end{equation}
	If ${W}^t_N>0$, the solution $\hat{\boldsymbol{{\omega}}}$ to (\ref{quadprog1})  is
	\begin{align}
	{\hat{\boldsymbol{\omega}}}&=
	\begin{bmatrix}
	\boldsymbol{e}_1^\prime&
	0
	\end{bmatrix}
	{\begin{bmatrix}
		{W}^t_N & \boldsymbol{e}_1\\
		\boldsymbol{e}_1^\prime& 0
		\end{bmatrix}}^{\text{-}1}
	\begin{bmatrix}
	\mathbf{0}\\
	1
	\end{bmatrix}\nonumber\\
	&=\frac{({W}^t_N)^{\text{-}1}\boldsymbol{e}_1}{\boldsymbol{e}_1^\prime{({W}^t_N)}^{\text{-}1}\boldsymbol{e}_1}
	\end{align} 
	which completes the proof.
\end{proof} 

  Lemma \ref{lemma4} shows that the least square estimate of the weights in (\ref{ioc_least}) can be achieved  incrementally by adding the new segment information from a new time step to the matrix ${W}^t_N$. As ${W}^t_N\in \mathbb{R}^{r\times r}$ is of fixed dimension, there is not additional memory consumption as new available data is included. Given previous data segments at each of the time steps, the least square estimate of the unknown weights  are solved by (\ref{recoveredsolution}).
Based on Lemma \ref{lemma4}, we present the IOC algorithm using demonstration segments in Algorithm \ref{algorithm}. 

\begin{algorithm2e}[h]
	\small \caption{Incremental IOC to recover objective}
	\label{algorithm}
	\SetKwInput{initialize}{Initialize}
	\KwIn{
% demonstration segments $\{ \boldsymbol{\xi}_{\ubar{t}_{i}:\bar{t}_{i}}: i=1,2,...,N\} $; \newline
		a  feature vector $\boldsymbol\phi$.}                                            
	\initialize{$W_N^0=\boldsymbol{0}$, $\mathcal{S}_t=\emptyset$}
% 	\While{$\boldsymbol{\hat{\omega}}$ does not converge}{
    \For{$t=1,2,\cdots$}{
        % $W_0^t=\boldsymbol{0}$;\newline
	    Append a new data segment $ \boldsymbol{\xi}_{\ubar{t}:\bar{t}} $ to the set $\mathcal{S}_{t}$;
	    
% 		\For{$i$th segment $\boldsymbol{\xi}_{\ubar{t}_{i}:\bar{t}_{i}}$}{
% 		\For{$j=1:M$}{

    		\eIf{$\boldsymbol{\xi}_{\ubar{t}:\bar{t}}$ is effective in \eqref{SDE}}{
    			${W}_N^t\leftarrow{W}_{N}^{t\text{-}1}+\boldsymbol{R}(\boldsymbol{\xi}_{\ubar{t}:\bar{t}})^\prime\boldsymbol{R}(\boldsymbol{\xi}_{\ubar{t}:\bar{t}})$}
    		{
    			${W}_N^t\leftarrow{W}_{N}^{t\text{-}1}+\boldsymbol{0}$
    		}
% 		}
% 	$W^k_N = W^{k-1}_N + W^k_j$ \newline
	Compute the least-square estimate $\boldsymbol{\hat{\omega}}$ via (\ref{recoveredsolution}).
	}
	
\end{algorithm2e}

\section{Numerical Experiments}
In this section, we evaluate the proposed method on a simulated robot arm and a 6-DoF quadrotor UAV system.

\subsection{Two-link robot arm}

\begin{figure}[h]
	\center
	\includegraphics[width=0.9\columnwidth]{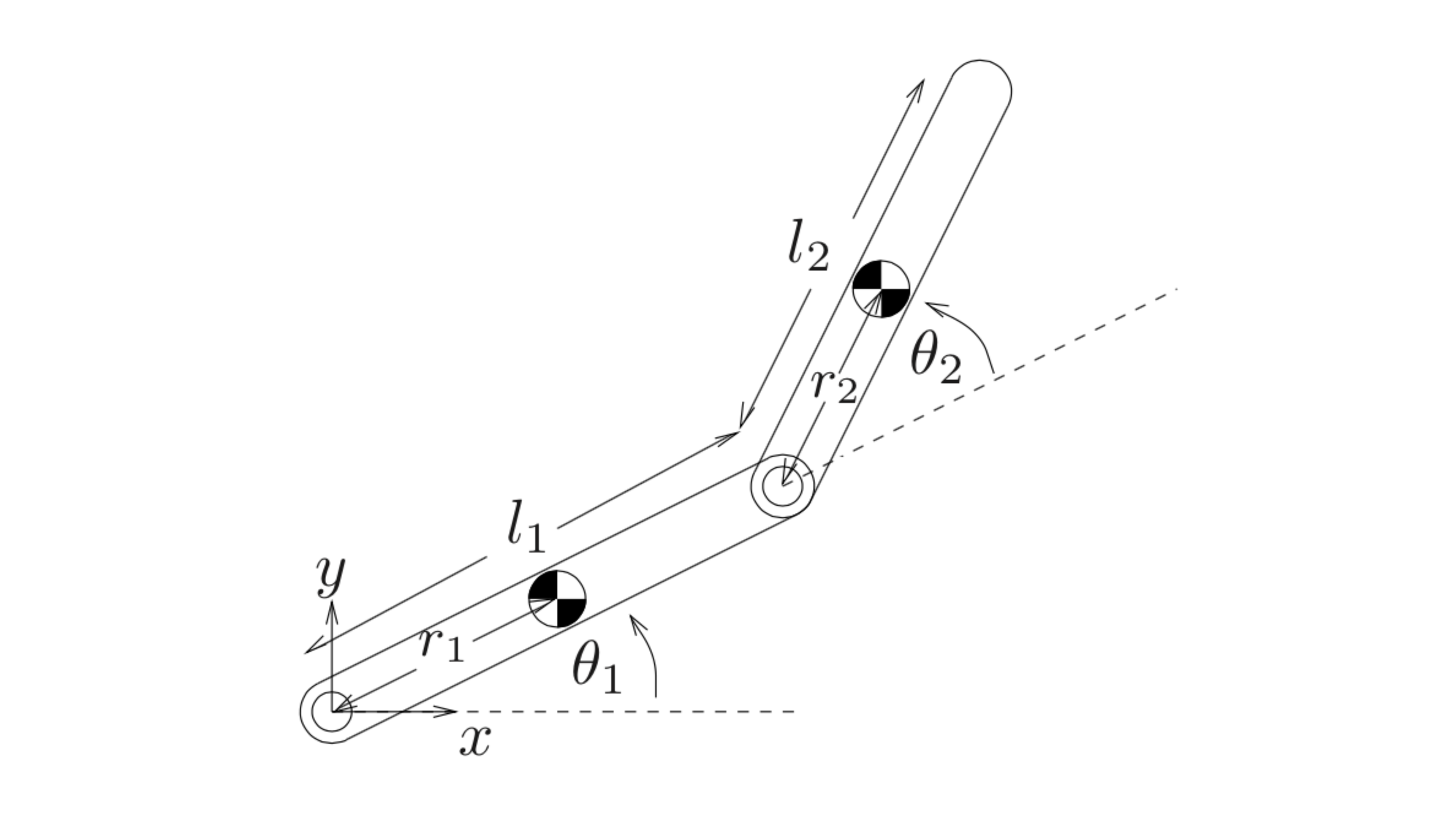}
	\caption{A simulated robot arm. }
	\label{graph}
\end{figure}

As shown in Fig. \ref{graph}, we consider that a two-link robot arm moves in vertical plane with  continuous dynamics given by \cite[p. 209]{spong2008robot}

\begin{equation} \label{armeDyn}
M(\boldsymbol\theta)\ddot{\boldsymbol\theta}+C(\boldsymbol\theta,\dot{\boldsymbol\theta})\dot{\boldsymbol\theta}+\boldsymbol g(\boldsymbol\theta)=\boldsymbol\tau,
\end{equation}
where $\boldsymbol\theta=[\theta_1, \theta_2]^{\prime} \in \mathbb{R}^2$ is the joint angle vector; $M(\boldsymbol \theta)\in \mathbb{R}^{2\times2}$ is the inertia matrix; $C(\boldsymbol\theta,\dot{\boldsymbol\theta}) \in \mathbb{R}^{2\times2}$ is the Coriolis matrix; $\boldsymbol g(\boldsymbol\theta) \in \mathbb{R}^2$ is the gravity vector; and $\boldsymbol\tau=[\tau_1, \tau_2]^{\prime}\in \mathbb{R}^2$ are the torques applied to each joint. The parameters used here follows \cite[p. 209]{spong2008robot}: the link mass $m_1=m_2=1 \mathrm{kg}$, the link length $l_1=l_2=1 \mathrm{m}$;  the distance from joint to center of mass (COM) $r_{1}=r_{2}=0.5 \mathrm{m}$, and the moment of inertia with respect to COM $I_1=I_2=1/12 \mathrm{kgm^2}$.	
By defining  the  states and control inputs of the robot arm system
\begin{equation}
\boldsymbol x\triangleq\begin{bmatrix}
\theta_1&\dot{\theta_1}&\theta_2&\dot{\theta_2}
\end{bmatrix}^{\prime}\quad \text{and} \quad \boldsymbol u\triangleq\boldsymbol\tau=
\begin{bmatrix}
\tau_1&\tau_2
\end{bmatrix}^{\prime},
\end{equation}
respectively, one could write (\ref{armeDyn}) in state-space representation  $\boldsymbol{\dot x}=\boldsymbol g(\boldsymbol x, \boldsymbol u)$  and further approximate it by the following discrete-time  form
\begin{equation}\label{equ_discretization}
\boldsymbol x_{t+1}\approx\boldsymbol x_t+\Delta\cdot\boldsymbol g(\boldsymbol x_t, \boldsymbol u_{t+1})\triangleq\boldsymbol{f}(\boldsymbol x_t, \boldsymbol u_{t+1}),
\end{equation}
where $\Delta =0.001 \mathrm{s}$ is the discretization interval. The motion of the robot arm  is controlled to  minimize the objective function (\ref{objFunc}), which here is set as a weighted distance to the goal state $\boldsymbol{x}^{\text{g}}=[{\theta}_1^\text{g}, {\dot{\theta}}_1^\text{g}, {\theta}_2^\text{g}, {\dot{\theta}}_2^\text{g}]^\prime=[0,0,0,0]^\prime$ plus the control effort $\norm{\boldsymbol{u}}^2$. Here, the corresponding  features and weights defined are as follows.
\begin{align}
\label{armOCObj}
    \boldsymbol{\phi}&= 
      \begin{bmatrix}
      ({\theta}_1-{\theta}_1^\text{g})^2\\
      ({{\dot\theta}}_1-{{\dot\theta}}_1^\text{g})^2\\
      ({\theta}_2-{\theta}_2^\text{g})^2\\
      ({{\dot\theta}}_2-{{\dot\theta}}_2^\text{g})^2\\
      \||\boldsymbol{u}||^2
    \end{bmatrix}, \qquad
    \boldsymbol{\omega}=\begin{bmatrix}
    1 \\ 2 \\ 1 \\ 1 \\1
    \end{bmatrix}, 
\end{align}
The initial condition of the robot arm is set as $ x_0=[\frac{2\pi}{3},0,\frac{-\pi}{2},0]^{\prime}$, and time horizon is set as $T=100$. We set the ground-truth weights as in (\ref{armOCObj}), and the resulting optimal trajectory of states and inputs is plotted in Fig.~\ref{trajectory}
\begin{figure}[h]
\hspace{-8pt}
	\includegraphics[width=0.9\columnwidth]{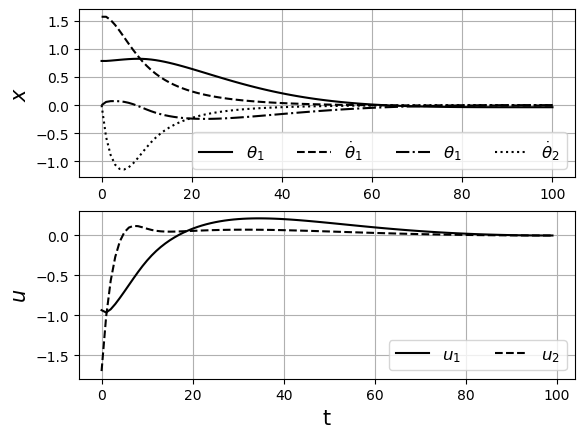}
	\caption{Optimal trajectory of the robot arm. }
	\label{trajectory}
\end{figure}

In the IOC task, we  learn the  weight vector $\boldsymbol{\omega}$ from  the segment data of the optimal trajectory in Fig. \ref{trajectory}. 
As shown in Table \ref{results_robotarm}, we perform five trials. For each trial, we will add a single data segment of the optimal trajectory to the set of data segments $\mathcal{S}_t$ at each time step, as indicated by the corresponding time intervals (second column). We apply Algorithm \ref{algorithm} to obtain the least-square estimate $\boldsymbol{\hat{\omega}}$ for each trial, and show the estimation results in the last column in Table~\ref{results_robotarm}.

\begin{table}[h]
\centering
	\caption	{IOC results from data segments}
	\begin{tabular}{lll}
		\toprule
		Trial No.  & Intervals of segments $[\ubar{t}_i, \bar t_i]$   & Estimate $\boldsymbol{\hat{\omega}}$     \\
		\midrule
		Trial 1 &    \begin{tabular}{@{}l@{}}$[1, 2]$,  $[10, 13]$,\\ $[70, 73]$, $[80, 83]$   \end{tabular}                   & $[1, 2, 1, 1, 1]$  \\[10pt]
		Trial 2 & $[10, 30],[50,51]$                          & $[1, 2, 1, 1, 1]$  \\[10pt]
		Trial 3 & $[50, 55],[90,92]$                          & $[1, 2, 1, 1, 1]$  \\[10pt]
		Trial 4 & \begin{tabular}{@{}l@{}} $[1, 4]$ \end{tabular}     &$[7.8, 1.5, 325.1, 1.4, 1]$\\[10pt]
		Trial 5 &\begin{tabular}{@{}l@{}}  $[1, 4]$, $[10, 13]$\end{tabular}      &$[1, 2, 1, 1, 1]$\\
		\bottomrule
	\end{tabular}
	\label{results_robotarm}
\end{table}

As shown in Table \ref{results_robotarm}, in all trials the algorithm successfully obtains the estimate $\boldsymbol{\hat{\omega}}$ to the feature weights in (\ref{armOCObj}) except in Trial 4. For all trials, we have randomly selected segments sparsely located in the time horizon. We have tested other data segments of the trajectory, and observed that most of trajectory segments are  effective except for those that are  very near the end of the trajectory, such as the segments within the time interval $[90, 100]$. This is because, as the system trajectory in Fig. \ref{trajectory} finally converges to zero,  the   states and inputs at the end of time horizon are very close to zeros (low-excitation) and thus likely become non-effective.

In Table \ref{results_robotarm}, it is also worth noting that  Trial 4 fails to recover the true weight vector. This is because   although the segment $\boldsymbol{\xi}_{1:4}$ in Trial 4 is data-effective (i.e., $\rank\boldsymbol{E}(\boldsymbol{\xi}_{1:4})=4 $), however, $\rank\boldsymbol{R}(\boldsymbol{\xi}_{1:4})=3<5$ and thus $\boldsymbol{R}(\mathcal{S}_t)$ has a kernel with dimension larger than one, which means a vector in the kernel is not guaranteed to be a scaled version of the true weight \cite{WanxinAutomatica}. To address this, we add another segment $\boldsymbol{\xi}_{10:13}$, as shown in Trial 5, in order to fulfill such rank requirement, and now $\rank\boldsymbol{R}(\mathcal{S}_t)=4$. Therefore, Trial 5 successfully estimates the true weight vector $\boldsymbol{\omega}$. 

Data-effectiveness is a precondition for a segment to be used for solving IOC problems. Although a single segment is data-effective, it may not necessarily suffice for recovering the weight. The aforementioned rank requirement  is stricter than the data-effectiveness condition  (\ref{SDE}), because $\boldsymbol{E}(\boldsymbol{\xi}_{\ubar{t}:\bar{t}})$ only relies on segment data and dynamics, while $\boldsymbol{R}(\boldsymbol{\xi}_{\ubar{t}:\bar{t}})$ additionally relies on features.

\subsection{Quadrotor UAV}
Next, we apply the proposed method to learn the objective function for a 6-DoF quadrotor UAV maneuvering system. Consider a quadrotor UAV with the following dynamics

\begin{equation}
    \begin{aligned} 
         \dot{\boldsymbol{p}}_I &= \boldsymbol{v}_I, \\
         m\dot{\boldsymbol{v}}_I &= m\mathbf{g}_I+\mathbf{F}_I, \\
         \dot{\boldsymbol{q}}_{B/I} &= \frac{1}{2}\Omega(\boldsymbol{\omega}_B)\boldsymbol{q}_{B/I}, \\
         J_B\dot{\boldsymbol{\omega}}_B &= \mathbf{M}_B - \boldsymbol{\omega} \times J_B\boldsymbol{\omega}_B.
    \end{aligned}
\end{equation}
Here, the subscription $_B$ and $_I$ denote a quantity is expressed in the body frame and inertial (world) frame, respectively; $m$ and $J_B\in\mathbb{R}^{3\times3}$ are the mass ($m=1$kg) and moment of inertia ($J_B=\diag[1,1,5]\text{kgm}^2$) with respect to body frame of the UAV, respectively. $g$ is the gravitational constant ($g=10\text{kg}/\text{m}^2$), $g_I=[0,0,g]^\prime$. $\boldsymbol{p}\in\mathbb{R}^3$  and $\boldsymbol{v}\in\mathbb{R}^3$ are the position and velocity vector of the UAV; $\boldsymbol\omega_B\in\mathbb{R}^3$ is the angular velocity vector of the UAV; $\boldsymbol{q}_{B/I}\in\mathbb{R}^4$ is the unit quaternion \cite{kuipers1999quaternions} that describes the attitude of UAV with respect to  the inertial frame; $\Omega(\boldsymbol\omega_B)$ is defined as:

\begin{equation}
    \Omega(\boldsymbol\omega_B) =
    \begin{bmatrix}
    0&-\omega_x&-\omega_y&-\omega_z \\
    \omega_x&0&\omega_z &-\omega_y \\
    \omega_y&-\omega_z &0&\omega_x \\
    \omega_z &\omega_y&-\omega_x&0
    \end{bmatrix},
\end{equation}
$\mathbf{M}_B\in\mathbb{R}^3$ is the torque applied to the UAV; $\mathbf{F}_I\in\mathbb{R}^3$ is the force vector applied to the UAV center of mass. The total force magnitude $f=\norm{\mathbf{F}_I}\in\mathbb{R}$ (along z-axis of the body frame) and torque $\mathbf{M}_B=[M_x,M_y,M_z]^\prime$ are generated by thrust from four rotating propellers $[T_1, T_2, T_3, T_4]^\prime$, their relationship can be expressed as:

\begin{equation}
    \begin{bmatrix}
    f\\
    M_x\\
    M_y\\
    M_z
    \end{bmatrix}
    =
    \begin{bmatrix}
    1&1&1&1 \\
    0&-l_w/2&0&l_w/2 \\
    -l_w/2&0&l_w/2&0 \\
    c&-c&c&-c
    \end{bmatrix}
    \begin{bmatrix}
    T_1\\
    T_2\\
    T_3\\
    T_4
    \end{bmatrix},
\end{equation}
where $l_w$ is the wing length of the UAV ($l_w = 0.4m$) and $c$ is a fixed constant ($c=0.01$). Similar to (\ref{equ_discretization}), we discretize the above dynamics with discretization interval of 0.1s. 

The state and input vectors of the UAV are defined as:
\begin{equation}
    \begin{aligned}
        \boldsymbol{x} &\triangleq
        \begin{bmatrix}
        \boldsymbol{p}'& \boldsymbol{v}'& \boldsymbol{q}' & \boldsymbol\omega'
        \end{bmatrix}
        ' \in \mathbb{R}^{13}, \\
        \boldsymbol{u} &\triangleq
        \begin{bmatrix}
        T_1&T_2&T_3&T_4
        \end{bmatrix}
        '\in\mathbb{R}^4.
    \end{aligned}
\end{equation}
The control objective function of the UAV includes a carefully selected attitude error term. As used in \cite{lee2010geometric}, we define the attitude error between UAV's current attitude $\boldsymbol{q}$ and the goal attitude $\boldsymbol{q}^{\text{g}}$ as:
\begin{equation}
    e(\boldsymbol{q},\boldsymbol{q}^\text{g}) = \frac{1}{2}\text{Tr}(I-R'(\boldsymbol{q}^\text{g})R(\boldsymbol{q})),
\end{equation}
where $R(\boldsymbol{q})\in\mathbb{R}^{3\times3}$ is the direction cosine matrix \cite{kuipers1999quaternions} directly corresponding to the quaternion $\boldsymbol{q}$. Other error terms that are included in the control objective function are simply the squared distances to their corresponding goals.

We generate the UAV optimal trajectory by minimizing a given control objective function. The initial state is set as $\boldsymbol{x}_0=[\boldsymbol{p}_\text{0},\boldsymbol{v}_\text{0},\boldsymbol{q}_\text{0},\boldsymbol{\omega}_\text{0}]^\prime=[-8,-6,9,0,0,0,1,0,0,0,1,1,1]'$, and the  goal state is set as $\boldsymbol{x}^\text{g}=[\boldsymbol{p}^\text{g},\boldsymbol{v}^\text{g},\boldsymbol{q}^\text{g},\boldsymbol{\omega}^\text{g}]^\prime=[0,0,0,0,0,0,1,0,0,0,0,0,0]'$. The control objective function is written as the weighted distance to the goal state plus the control effort $\norm{\boldsymbol{u}}^2$, where the features and weights are defined as follows:
\begin{equation}
    \boldsymbol{\phi} = 
    \begin{bmatrix}
    ||\boldsymbol{p}-\boldsymbol{p}^\text{g}||^2\\
    ||\boldsymbol{v}-\boldsymbol{v}^\text{g}||^2\\
    \frac{1}{2}\text{Tr}(I-R'(\boldsymbol{q}^\text{g})R(\boldsymbol{q}))\\
    ||\boldsymbol{u}||^2
    \end{bmatrix}, \qquad \boldsymbol{\omega} =
    \begin{bmatrix}
    2\\ 1\\ 1\\ 2
    \end{bmatrix}.
\end{equation}
The time horizon is set to $T=50$.

Similar to  the previous experiment, we set up four trials, and for each trial we observe different segments of the optimal state trajectories, as listed in the second column in Table \ref{results_UAV}. The result of feature weights estimation is shown the last column in  Table \ref{results_UAV}.

\begin{table}[h]
\centering
	\caption	{IOC results from data segments}
	\begin{tabular}{lll}
		\toprule
		Trial No.  & Intervals of segments $[\ubar{t}_i, \bar t_i]$    &  Estimate $\boldsymbol{\hat{\omega}}$     \\
		\midrule

		Trial 1 & $[5, 40]$                            &   $[2,1,1,2]$  \\[10pt]
		Trial 2 &    \begin{tabular}{@{}l@{}}$[5, 12]$,  $[10, 17]$,\\ $[25, 45]$ , $[20, 40]$   \end{tabular}     &  $[2,1,1.06,1.99]$  \\[10pt]
		Trial 3 & \begin{tabular}{@{}l@{}}$[9, 16]$,  $[26, 39]$ \end{tabular}   & $[2, 1,1.03,1.99]$\\[10pt]
		Trial 4 &\begin{tabular}{@{}l@{}}  $[1, 8]$,  $[12, 19]$ ,\\$[21, 41]$\end{tabular}   & $[2, 1,1,1.99]$\\
		\bottomrule
	\end{tabular}
	\label{results_UAV}
\end{table}

As shown in Table \ref{results_UAV}, in different trials, we use different segment of system trajectory to recover the true weight vector $\boldsymbol{\omega}$ incrementally. All used segments are data-effective. The proposed method successfully estimates  feature weights. The results demonstrate  effectiveness of the proposed method for incrementally learning an objective function.

\section{Conclusions and Future Directions}
In this paper, an incremental inverse optimal control method is proposed to learn the objective function. The available data is a collection of multiple segments of a system optimal trajectory from different time steps.  We first introduce the concept of data effectiveness to  evaluate the contribution of any segment to IOC, and then show that  each segment data can be utilized  to establish a linear constraint on the unknown objective weights.  Along this key idea,  the proposed IOC method incrementally incorporates each segment to obtain a least-square estimate of the  weights.

For  future research, we will extend the proposed method to a model-free IOC method. By say model free, it means that the dynamics model of the optimal control system is not known, and thus requires additional techniques for model approximation. The motivation here is that the assumption of a known dynamics model is sometimes challenging to fulfill since obtaining such dynamical model often requires expert knowledge. Data-driven methods would be considered as one of the possible options to recover the dynamics model from given data (e.g. states and input observations). Moreover, the estimation of weights vector with noisy data would also be one of the future research directions.

%% Use plainnat to work nicely with natbib. 
\bibliographystyle{IEEEtran}
\bibliography{references}

\addtolength{\textheight}{-12cm}   % This command serves to balance the column lengths

\end{document}